\documentclass{article}

\usepackage{microtype}
\usepackage{graphicx}
\usepackage{subfigure} 
\usepackage{booktabs} % for professional tables

\PassOptionsToPackage{hyphens}{url}
\usepackage{hyperref}

\usepackage[accepted]{icml2018}

\usepackage[inline]{enumitem}
\usepackage{amsmath}
\usepackage{amsfonts}
\usepackage{amssymb}
\usepackage{proof}
\usepackage{amsthm}
\usepackage{algorithm,algorithmic}
\usepackage{todonotes}

\usepackage{proof}

\newtheorem{theorem}{Theorem}
\newtheorem{lemma}{Lemma}

\newtheorem{proposition}{Proposition}
\newtheorem{definition}{Definition}

\newcommand{\norm}[1]{\left|\left|#1\right|\right|}

\icmltitlerunning{Transfer with Model Features}

\begin{document}

\twocolumn[
\icmltitle{Transfer with Model Features in Reinforcement Learning}




\begin{icmlauthorlist}
\icmlauthor{Lucas Lehnert}{ed}
\icmlauthor{Michael L. Littman}{ed}
\end{icmlauthorlist}

\icmlaffiliation{ed}{Computer Science Department, Brown University, Providence, RI, United States}

\icmlcorrespondingauthor{Lucas Lehnert}{lucas\_lehnert@brown.edu}

\icmlkeywords{Machine Learning, ICML}

\vskip 0.3in
]



\printAffiliationsAndNotice{}  % leave blank if no need to mention equal contribution

\begin{abstract}
A key question in Reinforcement Learning is which representation an agent can learn to efficiently reuse knowledge between different tasks.
Recently the Successor Representation was shown to have empirical benefits for transferring knowledge between tasks with shared transition dynamics.
This paper presents Model Features: a feature representation that clusters behaviourally equivalent states and that is equivalent to a Model-Reduction.
Further, we present a Successor Feature model which shows that learning Successor Features is equivalent to learning a Model-Reduction.
A novel optimization objective is developed and we provide bounds showing that minimizing this objective results in an increasingly improved approximation of a Model-Reduction.
Further, we provide transfer experiments on randomly generated MDPs which vary in their transition and reward functions but approximately preserve behavioural equivalence between states.
These results demonstrate that Model Features are suitable for transfer between tasks with varying transition and reward functions.
\end{abstract}

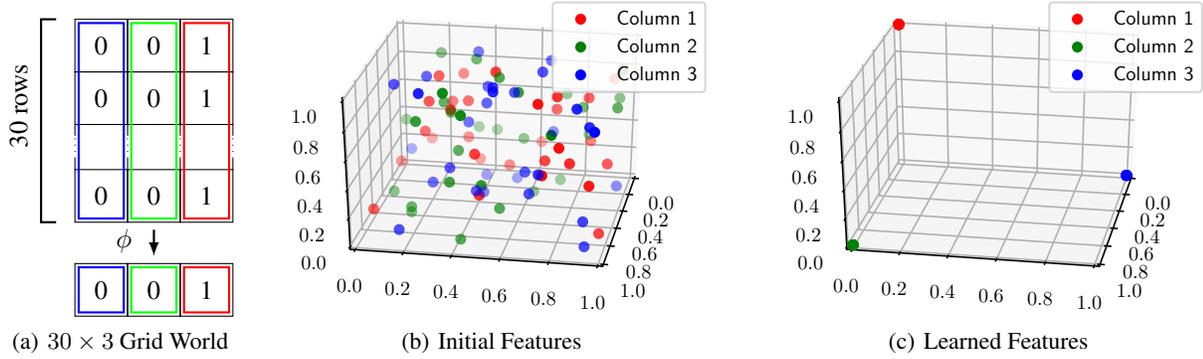
\begin{figure*}
\centering
\subfigure[$30 \times 3$ Grid World]{\label{fig:grid-world}%!TEX root = model_features.tex
%

\begin{tikzpicture}[domain=-2.5:4] 

% Left MDP

\node[](v) at ( .0, .0) [rectangle, draw, minimum size=.7cm] {0};
\node[](v) at ( .7, .0) [rectangle, draw, minimum size=.7cm] {0};
\node[](v) at (1.4, .0) [rectangle, draw, minimum size=.7cm] {1};

\draw[] (-.35, .35) -- (-.35, .5 );
\draw[] ( .35, .35) -- ( .35, .5 );
\draw[] (1.05, .35) -- (1.05, .5 );
\draw[] (1.75, .35) -- (1.75, .5 );

\draw[] (-.35,.80) -- (-.35,.95);
\draw[] ( .35,.80) -- ( .35,.95);
\draw[] (1.05,.80) -- (1.05,.95);
\draw[] (1.75,.80) -- (1.75,.95);

\draw[dotted] (-.35, .5 ) -- (-.35,.80);
\draw[dotted] ( .35, .5 ) -- ( .35,.80);
\draw[dotted] (1.05, .5 ) -- (1.05,.80);
\draw[dotted] (1.75, .5 ) -- (1.75,.80);

\node[](v) at ( .0,1.3) [rectangle, draw, minimum size=.7cm] {0};
\node[](v) at ( .7,1.3) [rectangle, draw, minimum size=.7cm] {0};
\node[](v) at (1.4,1.3) [rectangle, draw, minimum size=.7cm] {1};

\node[](v) at ( .0,2.0) [rectangle, draw, minimum size=.7cm] {0};
\node[](v) at ( .7,2.0) [rectangle, draw, minimum size=.7cm] {0};
\node[](v) at (1.4,2.0) [rectangle, draw, minimum size=.7cm] {1};

\node[rotate=90](v) at (-1.1,1.2) [] {30 rows};
\draw[thick] (-.6,-.35) -- (-.8,-.35) -- (-.8,2.35) -- (-.6,2.35); % was 2.75

\draw[thick,  blue] (-.3,-.3) -- (-.3,2.3) --  (.3,2.3) -- (.3,-.3)  -- cycle;
\draw[thick, green] ( .4,-.3) -- ( .4,2.3) --  (1.,2.3) -- (1.,-.3)  -- cycle;
\draw[thick,   red] (1.1,-.3) -- (1.1,2.3) -- (1.7,2.3) -- (1.7,-.3) -- cycle;

% abstract MDP

\node[](v) at ( .0,-1.25) [rectangle, draw, minimum size=.7cm] {0};
\node[](v) at ( .7,-1.25) [rectangle, draw, minimum size=.7cm] {0};
\node[](v) at (1.4,-1.25) [rectangle, draw, minimum size=.7cm] {1};

\draw[thick,  blue] (-.3,-.95) -- (-.3,-1.55) --  (.3,-1.55) -- (.3,-.95)  -- cycle;
\draw[thick, green] ( .4,-.95) -- ( .4,-1.55) --  (1.,-1.55) -- (1.,-.95)  -- cycle;
\draw[thick,   red] (1.1,-.95) -- (1.1,-1.55) -- (1.7,-1.55) -- (1.7,-.95)  -- cycle;

\draw[thick,-latex] (.7,-.45) -- (.7,-.8);
\node[](p) at (.3,-.6) [] {$\phi$};

\end{tikzpicture}}
\subfigure[Initial Features]{\label{fig:feat-init}\includegraphics[width=.38\textwidth]{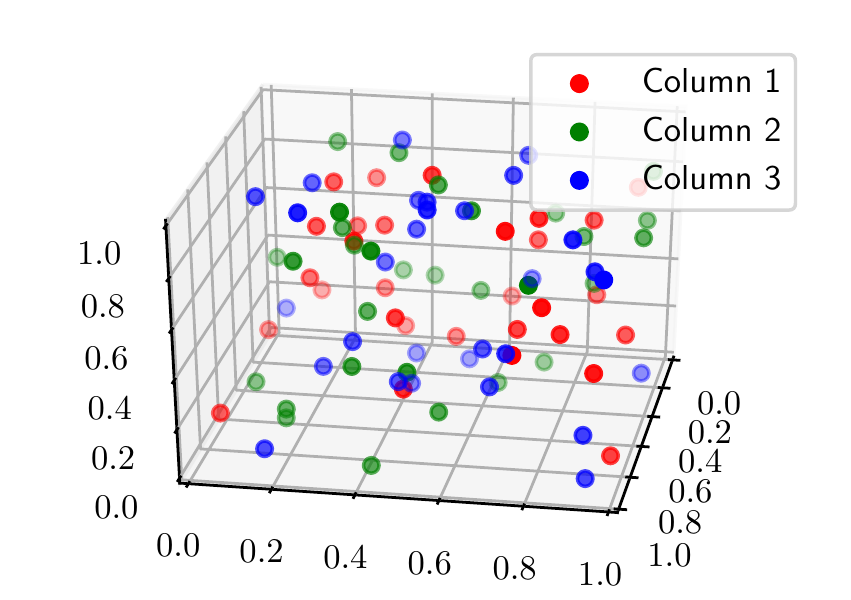}}
\subfigure[Learned Features]{\label{fig:feat-learned}\includegraphics[width=.38\textwidth]{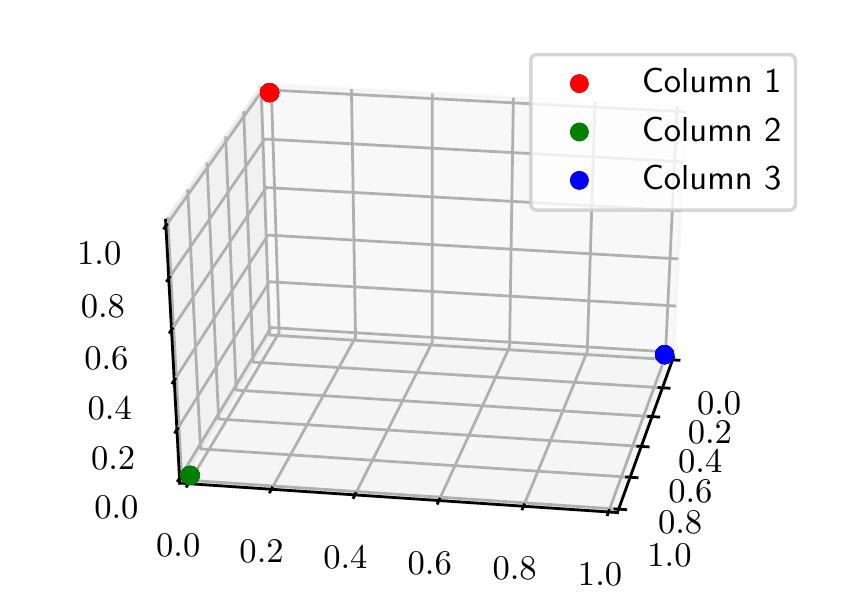}}

\caption{Grid World Example. The agent can move up, left, right, or down, and will always receive a +1 reward for selecting an action in the red column, and a zero reward otherwise. A model reduction collapses each column into a single state (bottom~\ref{fig:grid-world}). This three-state MDP captures all dynamics: the +1 reward state is distinct from the remaining states describe the distance to the positive reward state. The row index is not needed in order to evaluate an arbitrary policy. Our goal is optimize an initially random feature representation (Figure~\ref{fig:feat-init}) so that bisimilar features are assigned approximately the same feature vector (Figure~\ref{fig:feat-learned}).}
\label{fig:grid-world-example}

\end{figure*}

% \section{Introduction}

In Reinforcement Learning (RL)~\cite{sutton98,kaelbling1996reinforcement} one considers interactions between an intelligent agent and an environment.
These interactions consists of the agent choosing an action from a set of actions, which then triggers a transition in the environment's state.
For each transition the agent is provided with a single scalar reward.
The agent's objective is to compute an action-selection strategy, also called a \emph{policy}, that maximizes overall rewards.
Transfer experiments in RL~\cite{taylor2009transfer} can provide insight into which information or representation an agent should retain from a task in order to solve a different task more efficiently.
Recently, the Successor Representation~\cite{dayan1993successor}, which predicts the visitation frequency of future states, was shown to have empirical benefits in transfer experiments~\cite{baretto2017sf,zhang2016deepsucc,lehnert2017sf} and it was shown to be a representation humans are likely to use when transferring knowledge between different tasks~\cite{momennejad2017successor}.

The Successor Representation can be viewed as an intermediate between model-free and model-based RL~\cite{momennejad2017successor}.
In model-free RL, the intelligent agent only learns a value function which predicts the future return of a single policy.
In model-based RL, the intelligent agent learns a model of its environment which is sufficient to make predictions about individual future reward outcomes, given any arbitrary action sequence~\cite{sutton1990dyna}.
In comparison, Successor Representations (SRs) predict future state visitation frequencies under a particular policy.
By associating a feature vector with each state, SRs can be generalized to Successor Features (SFs), which predict the discounted sum of future state features~\cite{baretto2017sf}.
Because the value function of any policy can be written as the dot-product between the SF of a specific state as well as the reward model, transfer between tasks with different reward models is efficient~\cite{lehnert2017sf}.
However, \citet{lehnert2017sf} have also shown that SFs are tied to the transition function and a particular policy, making SFs unsuitable for transfer between tasks where more than the reward function changes.

In this paper we introduce \emph{Model Features}, a feature representation that provably assigns identical features to behaviourally equivalent states~\cite{givan2003bisimulation} (Section~\ref{sec:mr}).
Model Features can be viewed as a form of \emph{Model-Reduction} which compresses the state space such that future reward outcomes can be predicted using only the compressed representation.
Further, we present a modification of the architecture presented by~\citeauthor{baretto2017sf} and show that this architecture can be used to learn Model Features (Section~\ref{sec:sf} and ~\ref{sec:learning}).
Hence, the presented SF architecture is not restricted for transfer between tasks with common transition functions (Section~\ref{sec:transfer}).

\section{Model Reductions}\label{sec:mr}

A finite \emph{Markov Decision Processes (MDP)} is a tuple $M = \langle \mathcal{S}, \mathcal{A}, p, r, \gamma \rangle$, with a finite state space $\mathcal{S}$, a finite action space $\mathcal{A}$, a transition function $p(s,a,s') = \mathbb{P} \{ s' | s,a \}$, a reward function $r : \mathcal{S} \times \mathcal{A} \times \mathcal{S} \to \mathbb{R}$, and a discount factor $\gamma \in [0,1)$.
The transition and reward function can also be written in matrix or vector notation as a stochastic state-to-state transition matrix $\pmb{P}^a$ and an expected reward vector $\pmb{r}^a$ respectively.
A policy $\pi : \mathcal{S} \times \mathcal{A} \to [0,1]$ specifies the probabilities with which actions are selected at any given state (and $\sum_{a \in \mathcal{A}} \pi(s,a) = 1$).
If a policy $\pi$ is used, the transition function and expected rewards generated by this policy are denoted with $\pmb{P}^\pi$ and $\pmb{r}^\pi$ respectively.
The value function $V^\pi(s) = \mathbb{E} \left[ \sum_{t=1}^\infty \gamma^{t-1} r(s_t,a_t,s_{t+1}) \middle| s_0 = s, \pi \right]$ can also be written as a vector $\pmb{v}^\pi = \sum_{t=1}^\infty \gamma^{t-1} \left[ \pmb{P}^\pi \right]^{t-1} \pmb{r}^\pi$.
The action-conditional Q-function consists of a set of vectors $\{ \pmb{q}^a \}_{a \in \mathcal{A}}$ with
\begin{equation}
\pmb{q}^a = \pmb{r}^a + \gamma \pmb{P}^a \pmb{v}^\pi.
\end{equation}

A Model-Reduction~\cite{givan2003bisimulation} is a clustering of the state space $\mathcal{S}$ such that no information of the transition and reward functions relevant for reward prediction is lost.
Specifically, a Model-Reduction clusters behaviourally equivalent states.
Consider the example grid world shown in Figure~\ref{fig:grid-world}.
In this MDP, each column forms a set of behaviourally equivalent states because for each state partition two criteria are satisfied:
\begin{enumerate*}
\item the one-step rewards are the same, and
\item for two states of the same partition the distribution over clustered next states is identical. 
\end{enumerate*} 
This is the case in Figure~\ref{fig:grid-world} and one can observe that the compressed MDP retains all information necessary to predict future reward outcomes, because only the columns describe the distance in terms of time steps to the +1 reward.
Bisimilarity between two states can be defined as follows. 

\begin{definition}[Bisimulation]\label{def:bisimulation}
For a finite MDP $M = \langle \mathcal{S}, \mathcal{A}, p, r, \gamma \rangle$, two states $s$ and $\tilde{s}$ are bisimilar if and only if $\forall a \in \mathcal{A}$, 
\begin{enumerate}[nolistsep]
\item $r(s,a) = r(\tilde{s},a)$, and
\item $\forall s_\phi \in \mathcal{B}$, $\sum_{s' \in s_\phi} p(s,a,s') = \sum_{s' \in s_\phi} p(\tilde{s},a,s')$,
\end{enumerate} 
where $\mathcal{B}$ is the set of state partitions, where all states in each partition $s_\phi$ are bisimilar.
\end{definition}

This paper introduces Model Features, a function $\phi : \mathcal{S} \to \mathbb{R}^n$ such that for any two bisimilar states $s$ and $\tilde{s}$, $\phi(s) = \phi(\tilde{s})$. % have approximately the same feature vector.
Figures~\ref{fig:feat-init} and~\ref{fig:feat-learned} illustrate the idea of learning such a feature representation that maps bisimilar states to the same feature cluster.

\subsection{Mapping MDPs to MDPs with State Abstractions}

To derive an optimization objective to learn Model Features, we tie feature representations to the state aggregation framework presented by~\citet{li2006abstraction}.
In their framework the state space is clustered to obtain a new abstract MDP, where each state in the abstract MDP corresponds to a cluster of states in the original MDP.
If each state $s$ of an MDP $M$ is represented as a one-hot bit-vector $\pmb{s}$ of dimension $| \mathcal{S} |$, then we can think of a feature representation $\phi$ as mapping a one-hot state $\pmb{s}$ to abstract one-hot states $\pmb{s}_\phi$ of a lower dimension.
Hence a feature representation can be written as a partition matrix.

\begin{definition}[Partition Matrix]\label{def:partition-matrix}
Let $\phi$ be a feature representation compressing a state space $\mathcal{S}$ to a state space $\mathcal{S}_\phi$ where $|\mathcal{S}_\phi| \le |\mathcal{S}_\phi|$.
Then the partition matrix $\pmb{\Phi} \in \mathbb{R}^{|\mathcal{S}| \times |\mathcal{S}_\phi|}$ is a zero-one bit matrix with entries $\pmb{\Phi}(s,s_\phi) = \pmb{1}\left[ s_\phi = \phi(s) \right]$, where $\pmb{1} \left[ \cdot \right]$ is the indicator function.
\end{definition}

\begin{figure}
\centering

\begin{tikzpicture}[domain=-2.5:4] 

\node[](vphi) at (3.4,2.7) {$\pmb{q}^{\pi,a}_{\phi}$};
\node[] at (3.4,1.5) {$Q^\pi_\phi(s_{\phi,1},a)$};
\node[] at (3.4,1.0) {$Q^\pi_\phi(s_{\phi,2},a)$};
\node[] at (3.4,0.5) {$Q^\pi_\phi(s_{\phi,3},a)$};
\draw[thick] (2.5,.3) -- (2.4,.3) -- (2.4,1.7) -- (2.5,1.7);
\draw[thick] (4.3,.3) -- (4.4,.3) -- (4.4,1.7) -- (4.3,1.7);

\node[](v') at (-.3,2.7) {$\pmb{q}^{\pi,a}$};
\node[] at (-.3,2.0) {$Q^\pi(s_1,a)$};
\node[] at (-.3,1.5) {$Q^\pi(s_2,a)$};
\node[] at (-.3,1.0) {$Q^\pi(s_3,a)$};
\node[] at (-.3,0.5) {$Q^\pi(s_4,a)$};
\node[] at (-.3,0.0) {$Q^\pi(s_5,a)$};
\draw[thick] (-1.1,-.2) -- (-1.2,-.2) -- (-1.2,2.2) -- (-1.1,2.2);
\draw[thick] ( .5,-.2) -- ( .6,-.2) -- ( .6,2.2) -- ( .5,2.2);

\node[](phimat) at (1.5,2.7) {$\pmb{\Phi}$};

\draw[thick, green] (0.8,0.0) -- (2.2,0.5);
\draw[thick, green] (0.8,0.5) -- (2.2,1.0) -- (0.8,1.0);
\draw[draw=none, fill, green, fill opacity=0.2] (0.8,0.5) -- (2.2,1.0) -- (0.8,1.0) -- cycle;
\draw[thick, green] (0.8,1.5) -- (2.2,1.5) -- (0.8,2.0);
\draw[draw=none, fill, green, fill opacity=0.2] (0.8,1.5) -- (2.2,1.5) -- (0.8,2.0) -- cycle;

\draw[-latex] (vphi) -- (phimat);
\draw[-latex] (phimat) -- (v');

\end{tikzpicture}

\caption{The $\pmb{\Phi}$ matrix can be thought of as a projection operation mapping states in the clustered state space to states in the original state space.}
\label{fig:partition-matrix}
\end{figure}
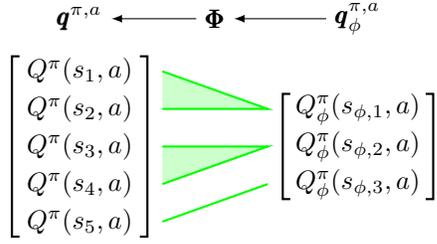

For example, suppose the feature representation $\phi$ is designed to represent the action-value function $Q^\pi$.
Then, the function $\phi$ has to merge states with equal Q-values into the same cluster, so that we can write
\begin{equation}
\pmb{q}^a = \pmb{\Phi} \pmb{q}_\phi^a, ~\forall a \in \mathcal{A}, \label{eq:model-free-features}
\end{equation}
where $\pmb{q}_\phi^a$ is a vector containing the Q-values each state cluster has in common.
Figure~\ref{fig:partition-matrix} illustrates how the partition matrix $\pmb{\Phi}$ clusters states schematically.
Further,~\eqref{eq:model-free-features} can be interpreted as a definition of a model-free feature representation, because the partition matrix $\pmb{\Phi}$ is designed to represent the value function $Q^\pi$.
The same reasoning can be applied to obtain conditions for Model Features.

\begin{theorem}\label{thm:bisim}
For an MDP $M=\langle \mathcal{S}, \mathcal{A}, p, r, \gamma \rangle$, let $\phi$ be a state abstraction that induces the abstract MDP $M_\phi = \langle \mathcal{S}_\phi, \mathcal{A}, p_\phi, r_\phi, \gamma \rangle$.
If the corresponding partition matrix $\pmb{\Phi}$ satisfies
\begin{equation}
\forall a \in \mathcal{A}, ~ \pmb{r}^a = \pmb{\Phi} \pmb{r}_\phi^a ~\text{and}~ \pmb{P}^a \pmb{\Phi}  =  \pmb{\Phi} \pmb{P}_\phi^a, \label{eq:thm-bisim}
\end{equation}
where $\pmb{r}^a$ and $\pmb{P}^a$ are the reward and transition matrices of $M$, and $\pmb{r}_\phi^a$ and $\pmb{P}_\phi^a$ are the reward and transition matrices of $M_\phi$, then $\pmb{\Phi}$ is a Model-Reduction (Definition~\ref{def:bisimulation}).
\end{theorem}
Intuitively, the conditions~\eqref{eq:thm-bisim} map an MDP $M$ into an MDP $M_\phi$ by ``compressing'' the reward and transition model.
The reward model is compressed in the same way as the Q-function in~\eqref{eq:model-free-features}.
The transition model is compressed such that the distribution over next state clusters (a column of $\pmb{P}^a \pmb{\Phi}$) equals the transition probabilities between clusters (a column of $\pmb{P}_\phi^a$) and copying this distribution to each state in the original MDP. 
If two states $s$ and $\tilde{s}$ are bisimilar, then $\pmb{s} \pmb{\Phi} \pmb{P}_\phi^a = \pmb{\tilde{s}} \pmb{\Phi} \pmb{P}_\phi^a$\footnote{Each row of $\pmb{\Phi} \pmb{P}_\phi^a$ contains a probability distribution over next clusters, and multiplying with the one-hot vector $\pmb{s}$ projects out the distribution for a state $s$.}, i.e. their distribution over next state clusters is identical.
Hence, the transition matrix identity in~\eqref{eq:thm-bisim} is identical to the transition model condition in Definition~\ref{def:bisimulation}.
The proof of Theorem~\ref{thm:bisim} is listed in Appendix~\ref{app:proofs}.

\section{Connection to Successor Features}\label{sec:sf}

The SR~\cite{dayan1993successor} is defined as the discounted sum of future states $\psi^\pi(s) = \mathbb{E} \left[ \sum_{t=1}^\infty \gamma^{t-1} \pmb{s}_t \middle| \pmb{s}_t = \pmb{s} \right]$.
In matrix notation, a SR for a particular policy $\pi$ can be written as 
\begin{equation}
\pmb{\Psi}^\pi_\text{SR} = \sum_{t=1}^\infty \gamma^{t-1} \left[ \pmb{P}^\pi \right]^{t-1}.
\end{equation}
Intuitively, a SR of a specific state describes the visitation frequency of all future states.
A column of a re-scaled SR $(1 - \gamma) \pmb{\Psi}^\pi$ then contains a marginal (over time steps) of reaching a specific state, where the number of time steps needed to reach a state follows a geometric distribution with parameter $\gamma$.
Let $\pmb{\Psi}^a$ be an action-conditional SR defined as
\begin{equation}
\pmb{\Psi}^a_\text{SR} = \pmb{I} + \gamma \pmb{P}^a \pmb{\Psi}^\pi_\text{SR}, \label{eq:sr-one-step}
\end{equation}
where $\pmb{\Psi}^a$ has a dependency on the policy $\pi$~\cite{lehnert2017sf}.
Let $\pmb{\Phi}$ be an arbitrary partition matrix and define a Successor Feature (SF) matrix $\pmb{\Psi}^a$ as
\begin{align}
\forall a \in \mathcal{A}, ~ \pmb{\Psi}^a = \pmb{\Psi}^a_\text{SR} \pmb{\Phi} % \\
% &= \left( \pmb{I} + \gamma \pmb{P}^a \pmb{\Psi}^\pi_\text{SR} \right) \pmb{\Phi} \\
&= \pmb{\Phi} + \gamma \pmb{P}^a \pmb{\Psi}^\pi_\text{SR} \pmb{\Phi} &\text{(by~\eqref{eq:sr-one-step})} \\
&= \pmb{\Phi} + \gamma \pmb{P}^a \pmb{\Psi}^\pi.
\end{align}
In our framework, each row of the matrix $\pmb{\Psi}^a_\text{SR} \pmb{\Phi}$ will describe the visitation frequency over state clusters, because the matrix $\pmb{\Phi}$ will aggregate all visitation frequency values over states that belong to the same partition (see also Figure~\ref{fig:partition-matrix}).
If we make the design assumption that 
\begin{align}
\forall a \in \mathcal{A}, ~ \pmb{\Psi}^a &= \pmb{\Phi} \pmb{F}^a, \label{eq:sf-phi}
\end{align}
then each row of $\pmb{\Psi}^a$ will duplicate a row of $\pmb{F}^a$ for states belonging to the same cluster.
Hence, the matrix $\pmb{F}^a$ is a SR over state clusters, rather than individual states.

By construction, each row of the expression $(1-\gamma) \pmb{\Psi}^a$ contains the marginal (over time steps) of reaching a specific state cluster, similar to the expression $\pmb{\Phi} \pmb{P}_\phi^a$ from~\eqref{eq:thm-bisim} where each row contains a probability distribution over next state clusters.
This connection allows us to show that SFs encode Model Features and thus Model-Reductions.

\begin{proposition}\label{prop:sf-bisim-equiv}
Consider a finite MDP $\mathcal{M} = \langle \mathcal{S}, \mathcal{A}, p, r, \gamma \rangle$, and a partition matrix $\pmb{\Phi}$.
Let $\overline{\pi}$ be an exploratory policy such that every possible transition in $M$ is visited with some probability.
If 
\begin{equation}
\forall a \in \mathcal{A}, ~ \pmb{r}^a = \pmb{\Phi} \pmb{r}_\phi^a ~\text{and}~\pmb{\Phi} \pmb{F}^a = \pmb{\Phi} + \gamma \pmb{P}^a \pmb{\Phi} \pmb{F}^{\overline{\pi}}, \label{eq:sf-bisim-bitmatrix}
\end{equation}
then $\forall a \in \mathcal{A}, ~ \pmb{r}^a = \pmb{\Phi} \pmb{r}_\phi^a ~\text{and}~ \pmb{P}^a \pmb{\Phi}  =  \pmb{\Phi} \pmb{P}_\phi^a$.
\end{proposition}

Intuitively, SFs are connected to Model Features because SFs can be viewed as a discounted infinite-step model.
In Proposition~\ref{prop:sf-bisim-equiv}, the policy $\overline{\pi}$ is assumed to be exploratory to ensure that all possible transitions are included in the SF.
Besides this assumption, the state representation does not depend on the policy used to compute the SFs, because our model does not condition the state representation $\pmb{\Phi}$ on the action space, in contrast to the model presented by~\citeauthor{baretto2017sf}.
The proof of Proposition~\ref{prop:sf-bisim-equiv} is listed in Appendix~\ref{app:proofs}.

\subsection{Approximate Model-Reductions}

To design an algorithm that approximates Model Features, we will now generalize the conditions stated in~\eqref{eq:sf-bisim-bitmatrix} to arbitrary feature representations.
For the remainder of the paper, the matrix $\pmb{\Phi}$ is assumed to be real valued and each row corresponds to a feature vector associated with a state.

The following theorem states that a feature representation that can represent one-step rewards as well as SFs, can also be used for representing the value function for any arbitrary policy.
This criterion is characteristic for a Model-Reduction~\cite{givan2003bisimulation,ferns2004bisimmetrics}, because the learned feature representation retains enough information to predict future rewards for any arbitrary action sequence.
The following theorem parallels Proposition~\ref{prop:sf-bisim-equiv} because it rephrases the conditions~\eqref{eq:sf-bisim-bitmatrix} to the approximate case for real valued matrices $\pmb{\Phi}$.

\begin{theorem}[Approximate Model Features]\label{thm:approx-model}
Consider a finite MDP $\mathcal{M} = \langle \mathcal{S}, \mathcal{A}, p, r, \gamma \rangle$ and a feature projection matrix $\pmb{\Phi}$.
%Further, let $\pmb{P}_\phi^a$ be a next feature projection matrix with $\norm{ \pmb{P}_\phi^a }_\infty \le 1$ and $\norm{\pmb{r}_\phi^a}_\infty \le 1$.
Assume that the SF of the feature space is $\pmb{F}^a = \pmb{I} + \gamma \pmb{P}_\phi^a \pmb{F}^{\overline{\pi}}$, with $|| \pmb{P}_\phi^a ||_\infty \le 1$ for every action $a$, and assume that for every action $a \in \mathcal{A}$,
\begin{align}
&\norm{ \pmb{\Phi} \pmb{r}_\phi^a - \pmb{r}^a }_\infty \le \varepsilon_r ~\text{and}~ \label{eq:r-err} \\
&\norm{ \pmb{\Phi} + \gamma \pmb{P}^a \pmb{\Phi} \pmb{F}^{\overline{\pi}} - \pmb{\Phi} \pmb{F}^a }_\infty \le \varepsilon_\psi. \label{eq:sf-err}
\end{align}
Then, for any policy $\pi$ with action-values $ \{ \pmb{q}^a \}_{a \in \mathcal{A}}$, 
\begin{equation}
\norm{\pmb{\Phi} \pmb{q}^a_\phi - \pmb{q}^a}_\infty \le \frac{\varepsilon_r}{1 - \gamma} + \frac{\varepsilon_\psi \left( 1 + \gamma \right) \norm{\pmb{r}_\phi}_\infty }{(1 - \gamma)^2}, \label{eq:err-bnd}
\end{equation}
where $|| \pmb{r}_\phi ||_\infty = \max_a || \pmb{r}_\phi^a ||_\infty$.
\end{theorem}

In comparison to Proposition~\ref{prop:sf-bisim-equiv}, Theorem~\ref{thm:approx-model} requires the matrix $\pmb{F}^a$ to be a SF with respect to some feature-to-feature transition matrix $\pmb{P}_\phi^a$.
This feature-to-feature transition matrix needs to satisfy $|| \pmb{P}_\phi^a ||_\infty \le 1$ in order to guarantee convergence of the discounted value function vectors $\pmb{q}_\phi^a$.
The proof of Theorem~\ref{thm:approx-model} is listed in Appendix~\ref{app:proofs}.

\section{Learning Model Features}\label{sec:learning}

Model Features can be approximated by minimizing the loss
\begin{align}
\mathcal{L}(\pmb{\Phi},\pmb{r}_\phi, \pmb{F}) &= \frac{1}{| \mathcal{A} |} \sum_{a} \Big( \norm{ \pmb{\Phi} \pmb{r}^a_\phi - \pmb{r}^a }_2^2  \nonumber \\
&\hspace{.2in} + \alpha \norm{ \pmb{\Phi} + \gamma \pmb{P}^a \pmb{\Phi} \pmb{F}^{\overline{\pi}} - \pmb{\Phi} \pmb{F}^a }_2^2 \Big) . \label{eq:loss-fn}
\end{align}
Using the L2 norm for training makes minimizing~\eqref{eq:loss-fn} easier with gradient optimization algorithms such as Adam~\cite{kingma2014adam}.
Given any particular solution to $\mathcal{L}$, the feature-to-feature transition matrix can be extracted with\footnote{using $\pmb{F}^a = \pmb{I} + \gamma \pmb{P}_\phi^a \pmb{F}^{\overline{\pi}}$ and $\pmb{F}^{\overline{\pi}} = \left( \pmb{I} - \gamma \pmb{F}^a \right)^{-1}$ } $\pmb{P}_\phi^a = \left( \pmb{F}^a - \pmb{I} \right) \left[ \pmb{F}^{\overline{\pi}} \right]^{-1} / \gamma$.
While performing gradient descent directly on $\mathcal{L}$ will cluster the feature space (see Figure~\ref{fig:grid-world-example} for example), the resulting feature representation will not necessarily produce a feature-to-feature transition matrix with $||  \pmb{P}_\phi^a ||_\infty \le 1$.
To obtain such a feature representation, the feature space is clustered using k-Means~\cite{bishop2016pattern} and the found centroids the feature space is projected into an approximately orthogonal set.
The projection step is performed by assembling a square projection matrix $\pmb{M}$ with each column being set to one of the centroid vectors. 
This matrix is then used to project the feature representation, the feature-to-feature transition matrix $\pmb{P}_\phi^a$, and feature reward model $\pmb{r}_\phi^a$.
Algorithm~\ref{alg:opt} outlines this procedure.
We found that our objective function produces feature clusters that form an linearly independent set.
Hence the matrix $\pmb{M}$ is invertible at every projection step.

\begin{algorithm}     
\caption{Learning Model Features}
\label{alg:opt}     
\begin{algorithmic}
\STATE Initialize $\pmb{\Phi}$, $\pmb{r}_\phi$, and $\pmb{F}$. Let $n$ be the number of features.
\LOOP 
\STATE Perform $k$ gradient updates on $\mathcal{L}$ w.r.t. $\pmb{\Phi}$, $\pmb{r}_\phi$, and $\pmb{F}$
\STATE k-Means clustering on row-space of $\pmb{\Phi}$ with $k=n$.
\STATE Construct $\pmb{M}$ with columns equal to cluster centroids.
\STATE $\pmb{\Phi} \leftarrow \pmb{\Phi} \pmb{M}^{-1}$
\STATE $\pmb{r}_\phi^a \leftarrow \pmb{M} \pmb{r}_\phi^a$ $\forall a \in \mathcal{A}$
\STATE $\pmb{F}^a \leftarrow \pmb{M} \pmb{F}^a \pmb{M}^{-1}$ $\forall a \in \mathcal{A}$
\ENDLOOP
\end{algorithmic}
\end{algorithm}

To compute the value error $|| \pmb{\Phi} \pmb{v}_\phi^\pi - \pmb{v}^\pi ||_\infty$ a policy $\pi$ is evaluated using only the feature transition and reward models $\pmb{P}_\phi^a$ and $\pmb{r}_\phi^a$ (Algorithm~\ref{alg:pol-eval}).
Only the learned model is used to make predictions about future reward outcomes.
The value function $\pmb{v}^\pi = \sum_a \pmb{\Pi}^a \pmb{q}^a$ where $\pmb{\Pi^a} = \text{diag} \{ \pi(s,a) \}_s$.

\begin{algorithm}     
\caption{Feature Policy Evaluation}
\label{alg:pol-eval}     
\begin{algorithmic}
\STATE Given $\pmb{P}_\phi^a$, $\pmb{r}_\phi^a$, and policy matrices $\pmb{\Pi}^a$ $\forall a$.
\REPEAT
\STATE $\pmb{q}_\phi^a \leftarrow \pmb{r}_\phi^a + \gamma \pmb{P}_\phi^a \pmb{v}^\pi$ $\forall a \in \mathcal{A}$
\STATE $\pmb{v}_\phi^\pi \leftarrow \pmb{\Phi}^+ \sum_{a} \pmb{\Pi}^a \pmb{\Phi} \pmb{q}_\phi^a$ \COMMENT{$\pmb{\Phi}^+$ is the pseudo-inverse}
\UNTIL{$\pmb{v}_\phi^\pi$ converges}
\end{algorithmic}
\end{algorithm}

We tested our implementation on the grid world shown in Figure~\ref{fig:grid-world-example}.
Figure~\ref{fig:feat-init} shows the initial feature representation and Figure~\ref{fig:feat-learned} shows the learned representation.
Figure~\ref{fig:train-bnds} plots the error bound~\eqref{eq:err-bnd} together with the computed prediction error $|| \pmb{\Phi} \pmb{v}_\phi^\pi - \pmb{v}^\pi ||_\infty$ for three policies: the optimal policy, uniform random action selection, and an $\varepsilon$-greedy policy which selects the optimal action with $\varepsilon=0.5$ probability.
The error bound can be computed only after 40000 iterations, because to evaluate a policy with Algorithm~\ref{alg:pol-eval} the feature transition matrix needs to satisfy $|| \pmb{P}_\phi^a ||_\infty \le 1$.
Further, the y-axis of Figure~\ref{fig:train-bnds} is scaled to a range between zero and ten, which is the range of possible values for the tested grid world when a discount factor $\gamma = 0.9$ is used. 
In Figure~\ref{fig:train-bnds} the value error of uniform random action selection is lowest and the optimal policy has often the highest error.
This is not surprising, because the model is only approximate and the SFs are trained to predict the future state visitation under uniform random action selection, which seems to incorporate approximation errors tied to this policy.
However, by Theorem~\ref{thm:approx-model} this dependency vanishes as the approximation errors ($\varepsilon_\psi$ and $\varepsilon_r$) tend to zero.

\begin{figure}
\centering
\includegraphics[width=.95\linewidth]{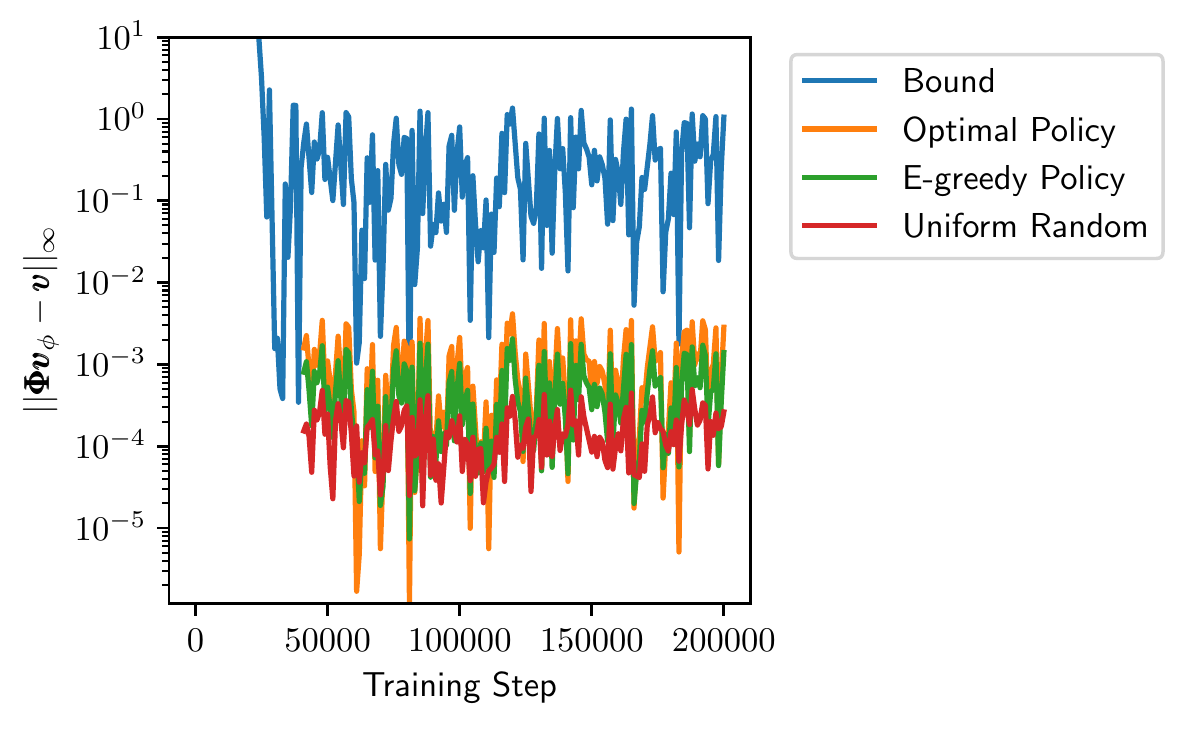}
\caption{Evolution of the value error during training on the grid world shown in Figure~\ref{fig:grid-world-example}.
For all experiments, the constant $\alpha = 0.001$ in~\eqref{eq:loss-fn} and the Adam optimizer with a learning rate of $10^{-3}$ was used. Every 40000 steps a k-means projections step was performed for the first 100000 gradient updates. All parameters were initialzed by sampling the interval $[0,1]$ uniformly.}
\label{fig:train-bnds}
\end{figure}

\section{Transfer Experiments}\label{sec:transfer}

\begin{figure}
\centering
\subfigure[Common $\pmb{\Phi}_\text{gt}$]{\label{fig:transfer-phi-same}\includegraphics[width=.49\linewidth]{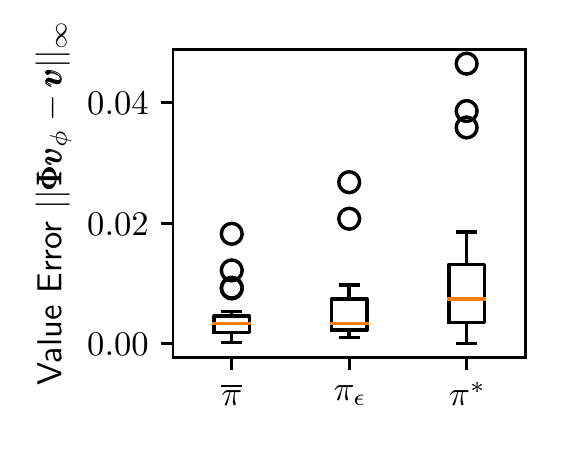}}~~~
\subfigure[Randomly modified $\pmb{\Phi}_\text{gt}$]{\label{fig:transfer-phi-same}\includegraphics[width=.49\linewidth]{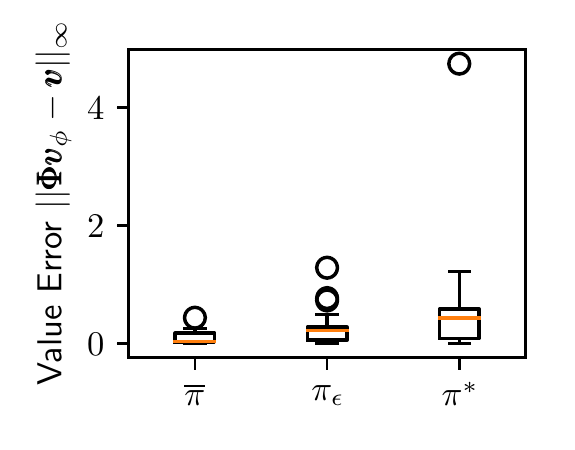}}
\caption{Transfer Value Prediction Errors for uniform action selection ($\overline{\pi}$), $\varepsilon$-Greedy ($\pi_\varepsilon$), and the optimal policy ($\pi^*$). 
Each model was trained for 30000 gradient updates with a 0.1 learning rate.}
\end{figure}

Model-Reductions and Model Features can be thought of as encoding information about which states a behaviourally equivalent in an MDP.
In this section we conduct two experiments on random MDPs to illustrate to what extend this information can be transferred between MDPs.
Specifically, our experiments demonstrate that Model Features are suitable for transfer between MDPs with different reward and transition functions assuming that behavioural equivalence between states is at least approximately preserved.

The first experiment randomly generates a ground truth partition matrix $\pmb{\Phi}_\text{gt}$ of size $50 \times 5$ (i.e. 50 ground level states and 5 state clusters) and constructs an MDP by randomly sampling feature transition and reward models $\pmb{P}_\phi^a$ and $\pmb{r}_\phi^a$\footnote{Rewards were generated by setting entries in $\pmb{r}_\phi^a$ to one with probability 0.1 and otherwise to zero. The matrices $\pmb{P}_\phi^a$ were sampled uniformly from $[0,1]$ and normalized to a stochastic matrix.}.
The resulting MDP $M = \langle \mathcal{S}, \mathcal{A}, p, r, \gamma \rangle$ with 50 states and four actions is then fed into Algorithm~\ref{alg:opt} to estimate a Model Feature matrix $\widehat{\pmb{\Phi}}$.
Then, using the same ground truth partition matrix $\pmb{\Phi}_\text{gt}$, 20 randomly sampled MDPs are constructed.
For each MDP, the same loss function $\mathcal{L}( \widehat{\pmb{\Phi}},\pmb{r}_\phi, \pmb{F})$ from~\eqref{eq:loss-fn} was minimized with respect to only the feature model $\pmb{r}_\phi^a$ and $\pmb{F}^a$.
The feature representation $\widehat{\pmb{\Phi}}$ was preserved for each of the 20 MDPs.
Figure~\ref{fig:transfer-phi-same} plots the value errors for the same three test policies used in the previous section.
Note that rewards were either zero or one, hence state values lie within the interval $[0, 10]$ (for a discount factor $\gamma = 0.9$).
Hence the prediction errors shown in Figure~\ref{fig:transfer-phi-same} are comparably low.
On the training MDP the feature representation $\widehat{\pmb{\Phi}}$ had an error bound value of $2.02 \cdot 10^{-4}$.
The fact that this bound is lower than the errors in Figure~\ref{fig:transfer-phi-same} indicate that approximation errors incorporate some information specific to the MDP on which the feature representation $\widehat{\pmb{\Phi}}$ was initially trained on.

The second experiment used the same protocol but the sampling process of the 20 random transfer MDPs was modified.
For each of the 20 transfer MDPs the ground truth partition matrix $\pmb{\Phi}_\text{gt}$ was first modified by randomly moving one ground state to a different cluster, otherwise the experiment is identical to the previous transfer experiment.
Figure~\ref{fig:transfer-phi-same} shows significantly higher value errors for the three different test policies.
This result is expected, because the model-feature representation $\widehat{\pmb{\Phi}}$ was trained for a different state partitioning.
However, despite the change in $\pmb{\Phi}_\text{gt}$, value errors are still comparably low.
In comparison, randomly predicting values could result in errors as high as 10, but most error values in Figure~\ref{fig:transfer-phi-same} fall below 2.0 for all 20 sampled MDPs with significantly lower averages for all tested policies.

\section{Discussion}

In RL, one question central to transfer is which information can be reused between different MDPs.
Recent work presented SFs as a useful representation for transfer between tasks with a shared transition function but varying reward function~\cite{baretto2017sf}.
Our results show that a modification of the SF model presented by~\citeauthor{baretto2017sf} can be used to learn Model Features.
This result implies that SFs not only encode information about visitation frequencies but also encode information about which states are behaviourally equivalent.
By isolating Model Features as a representation conditions only on the state space, our SF model is suitable for transfer between tasks with different reward and transition functions.
The only assumption we make is that states that are behaviourally equivalent in one task are also behaviourally equivalent in another task.
Further, by Theorem~\ref{thm:approx-model} we show how this assumption can be relaxed by considering approximations of Model Features.
Our model may give an explanation as to why SFs are beneficial when transferring information between MDPs when an underlying feature representation, similar to the $\pmb{\Phi}$ matrix in our model, is represented with a deep neural network~\cite{zhang2016deepsucc}.

Recently~\citet{ruan2015representation} presented algorithms to construct an approximate clustering of bisimilar states.
Their method relies on bisimulation metrics~\cite{ferns2004bisimmetrics} which use the Wasserstein metric to compress transition models.
We phrase learning a Model-Reduction as learning a feature representation.
This approach allows us to tie Model-Reduction to SFs and define a loss objective that can be optimized with a form or projected gradient descent.

\section{Conclusion}

This paper ties learning SFs to learning Model-Reductions.
By expressing Model-Reductions as Model Features, we derive a new SF model and loss objective to inform the design of an optimization algorithm to learn approximate Model-Reductions.
Further, we present a value error bound that can also be used to score how well the Model Features of one task can be transferred to another task.

Because Model Features only encode information about which states are equivalent for predicting rewards, our model is suitable for transfer between tasks with different transition and reward functions.
How well a particular feature representation can be reused depends on the equivalence between states being approximately preserved.

\appendix

\section{Proofs of Theoretical Results}\label{app:proofs}

\begin{definition}[Abstract MDP~\cite{li2006abstraction}]~\label{def:abstraction}
For a finite state-action MDP $M = \langle \mathcal{S}, \mathcal{A}, p, r, \gamma \rangle$, a fixed state abstraction function $\phi: \mathcal{S} \to \mathcal{S}_\phi$, and an arbitrary partition weighting function $\omega : \mathcal{S} \to [0,1]$, such that for every $s_\phi \in \mathcal{S}_\phi$, $\sum_{s \in \phi^{-1}(s_\phi)} \omega(s) = 1$, the abstract MDP $M_\phi = \langle \mathcal{S}_\phi, \mathcal{A}, p_\phi, r_\phi, \gamma \rangle$ is defined as
\begin{align*}
& p_\phi(s_\phi,a,s'_\phi) = \sum_{s \in \phi^{-1}(s_\phi)} \sum_{s' \in \phi^{-1}(s'_\phi)} \omega(s) p(s,a,s') , \\
& r_\phi(s_\phi, a) = \sum_{s \in \phi^{-1}(s_\phi)} \omega(s) r(s,a).
\end{align*} 
\end{definition}

Because this framework also uses a weighing function $\omega$, we define a \emph{weight matrix} in the following way. 

\begin{definition}[Weight Matrix]\label{def:weight-mat}
For a finite state-action MDP $M = \langle \mathcal{S}, \mathcal{A}, p, r, \gamma \rangle$, and consider a state abstraction function $\phi$ with weighting function $\omega$.
Further, assume $| \mathcal{S} | = n$ and $| \mathcal{S}_\phi | = m \le n$. 
We define the weight matrix $\pmb{\Omega}$ as $\pmb{\Omega}(s_\phi,s) = \mathbf{1}_{[ s_\phi =\phi(s) ]} \omega(s)$.
\end{definition}

This weight matrix $\pmb{\Omega}$ can be thought of as a left-inverse of the partition matrix because it projects the original state space $\mathcal{S}$ into the aggregated state space $\mathcal{S}_\phi$.
% Using these matrices an abstract MDP as in Definition~\ref{def:abstraction} can be constructed.

\begin{lemma} \label{lem:abs-mat}
Let $\phi$ be an abstraction with weighting function $\omega$ for a finite state-action MDP $M = \langle \mathcal{S}, \mathcal{A}, T, R, \gamma \rangle$.
Then the reward vector and transition matrix of the abstract MDP $M_\phi$ can be written as $\pmb{r}_{M_\phi}^a = \pmb{\Omega} \pmb{r}^a$ and $\pmb{P}_{M_\phi}^a = \pmb{\Omega} \pmb{P}^a \pmb{\Phi}$.
Further we have that $\pmb{\Omega} \pmb{\Phi} = \pmb{I}$.
\end{lemma}

\begin{proof}% [Proof of Lemma~\ref{lem:abs-mat}]
For the reward vector identity we have 
\begin{align*}
\pmb{\Omega} \pmb{r}^a (s_\phi) &= \sum_{s \in \mathcal{S}} \pmb{\Omega}(s_\phi, s) \pmb{r}^a(s) = \sum_{s \in \mathcal{S}} \omega(s) \mathbf{1}_{[ s_\phi =\phi(s) ]} r(s,a) \\
&= \sum_{s \in \phi^{-1}(s_\phi)} \omega(s) r(s,a) = r_\phi(s_\phi,a),
\end{align*}
hence $\pmb{r}_\phi^a = \pmb{\Omega} \pmb{r}^a$.
Similarly, we have that $\pmb{\Omega} \pmb{r}^a (s_\phi) = r_{M_\phi}(s_\phi,\pi(s_\phi))$ and thus $\pmb{r}_\phi^\pi = \pmb{\Omega} \pmb{r}^\pi_M$.
For the transition matrix identity we first look at the first matrix product:
\begin{align*}
\pmb{\Omega} \pmb{P}^a (s_\phi, s') &= \sum_{s \in \mathcal{S}} \pmb{\Omega}(s_\phi, s) \pmb{P}^a (s, s') \\
&= \sum_{s \in \mathcal{S}} \mathbf{1}_{[ s_\phi =\phi(s) ]} \omega(s) \pmb{P}^a (s, s') \\
&= \sum_{s \in \phi^{-1}(s_\phi)} \omega(s) p(s, a, s') .
\end{align*}
For the whole product we have:
\begin{align*}
& [ \pmb{\Omega} \pmb{P}^a \pmb{\Phi} ](s_\phi,s_\phi') = \sum_{s' \in \mathcal{S}} [\pmb{\Omega} \pmb{P}^a] (s_\phi, s') \pmb{\Phi}(s',s_\phi') \\
&= \sum_{s' \in \mathcal{S}} \sum_{s \in \phi^{-1}(s_\phi)} \omega(s) p(s, a, s') \mathbf{1}_{[ s_\phi' =\phi(s') ]} \\
&= \sum_{s \in \phi^{-1}(s_\phi)} \sum_{s' \in \phi^{-1}(s_\phi')} \omega(s) p(s, a, s')  = p_\phi(s_\phi,a,s_\phi').
\end{align*}
Hence $\pmb{P}_\phi^a = \pmb{\Omega} \pmb{P}^a \pmb{\Phi}$.
For the pseudo-inverse we have 
\begin{align*}
[\pmb{\Omega} \pmb{\Phi}](s_\phi,s_\phi) &= \sum_{s \in \mathcal{S}} \pmb{\Omega}(s_\phi,s) \pmb{\Phi}(s,s_\phi) \nonumber \\
&= \sum_{s \in \mathcal{S}} \omega(s) \mathbf{1}_{[ s_\phi =\phi(s) ]} = 1, \\
[\pmb{\Omega} \pmb{\Phi}](s_\phi,s_\phi') &= \sum_{s \in \mathcal{S}} \pmb{\Omega}(s_\phi,s) \pmb{\Phi}(s,s_\phi') \nonumber \\
&= \sum_{s \in \mathcal{S}} \omega(s) \mathbf{1}_{[ s_\phi =\phi(s) ]} \mathbf{1}_{[ s_\phi' =\phi(s) ]} = 0,
\end{align*}
hence $\pmb{\Omega} \pmb{\Phi} = \pmb{I}$.
\end{proof}

Definition~\ref{def:bisimulation} can be rewritten in the framework of~\citeauthor{li2006abstraction}:

\begin{definition}[Model Reduction]~\label{def:bisim}
For a finite state-action MDP $M = \langle \mathcal{S}, \mathcal{A}, p, r, \gamma \rangle$, the state abstraction function $\phi: \mathcal{S} \to \mathcal{S}_\phi$ is a model reduction if for two arbitrary states $s_1$ and $s_2$, $\phi(s_1) = \phi(s_2)$ if and only if
\begin{align*}
& \forall a,~ r(s_1,a) = r(s_2,a) ~~\text{and} \\
& \forall a, \forall s_\phi',~ \sum_{s' \in \phi^{-1}(s_\phi')} p(s_1,a,s') = \sum_{s' \in \phi^{-1}(s_\phi')} p(s_2,a,s'),
\end{align*}
where $\phi^{-1}(s_\phi)$ is the set of states mapped to $s_\phi$.
\end{definition}

\begin{proof}[Proof of Theorem~\ref{thm:bisim}]
Theorem~\ref{thm:bisim} can be proven by showing that $\pmb{r}^a = \pmb{\Phi} \pmb{r}_\phi^a$ and $\pmb{\Phi} \pmb{F}^a = \pmb{\Phi} + \gamma \pmb{P}^a \pmb{\Phi} \pmb{F}^{\overline{\pi}}$ implies that the partition matrix $\pmb{\Phi}$ clusters states according to Definition~\ref{def:bisim}.
For the reward model condition, we have
\begin{align}
\pmb{r}^a_M(s) &= \left( \pmb{\Phi} \pmb{r}_{M_\phi}^a \right)(s) \\
&= \sum_{s_\phi \in \mathcal{S}_\phi} \pmb{\Phi}(s,s_\phi) \pmb{r}_{M_\phi}^a(s_\phi) \\
&= \sum_{s_\phi \in \mathcal{S}_\phi} \mathbf{1}_{[ s_\phi =\phi(s) ]} \sum_{s \in \mathcal{S}} \omega(s) \pmb{r}_{M}^a(s) \\
&= \sum_{s \in \phi^{-1}(\phi(s))} \omega(s) \pmb{r}_{M}^a(s). \label{eq:thm-1-proof-1}
\end{align}
Note that~\eqref{eq:thm-1-proof-1} sums over all entries that lie in the same partition as $s$, but for an arbitrary choice of weighting function $\omega$.
Particularly,~\eqref{eq:thm-1-proof-1} also has to hold for a weighting function that has a weight of one on arbitrary state $s \in \phi^{-1}(\phi(s))$.
Hence $\forall s,s' \in \phi^{-1}(\phi(s)),~ \pmb{r}_{M}^a(s)=\pmb{r}_{M}^a(s')$.
For the transition model we have 
\begin{align*}
\pmb{P}_M^a \pmb{\Phi} = \pmb{\Phi} \pmb{P}_{M_\phi}^a &\iff \pmb{\Omega} \pmb{P}_M^a \pmb{\Phi} = \pmb{\Omega} \pmb{\Phi} \pmb{P}_{M_\phi}^a \\
&\iff \pmb{\Omega} \pmb{P}_M^a \pmb{\Phi} = \pmb{P}_{M_\phi}^a.
\end{align*}
Further, $\pmb{P}_M^a \pmb{\Phi}  =  \pmb{\Phi} \pmb{P}_{M_\phi}^a$ holds for any arbitrary weighting function, so for two distinct weighting functions or matrices $\pmb{\Omega}$ and $\tilde{\pmb{\Omega}}$, we have $\pmb{\Omega} \pmb{P}_M^a \pmb{\Phi} = \tilde{\pmb{\Omega}} \pmb{P}_M^a \pmb{\Phi}$.
By Lemma~\ref{lem:abs-mat}, for any $s_\phi$ and $s_\phi'$ this is equivalent to 
\begin{align}
&\sum_{s \in \phi^{-1}(s_\phi)} \sum_{s' \in \phi^{-1}(s'_\phi)} \omega(s) p(s,a,s') \nonumber \\
&~~= \sum_{s \in \phi^{-1}(s_\phi)} \sum_{s' \in \phi^{-1}(s'_\phi)} \tilde{\omega}(s) p(s,a,s')
\end{align}
Again, we pick two different weighting functions $\omega$ and $\tilde{\omega}$ that assign a weight of one to two different states in $\phi^{-1}(s_\phi)$, hence $\forall s,\tilde{s} \in \phi^{-1}(s_\phi)$, $\sum_{s' \in \phi^{-1}(s'_\phi)} p(s,a,s') = \sum_{s' \in \phi^{-1}(s'_\phi)} p(\tilde{s},a,s')$.
\end{proof}

\begin{proof}[Proof of Proposition~\ref{prop:sf-bisim-equiv}]
Assume that $\pmb{\Phi}$ is a partition matrix satisfying $\pmb{r}^a = \pmb{\Phi} \pmb{r}_\phi^a$ and $\pmb{\Phi} \pmb{F}^a = \pmb{\Phi} + \gamma \pmb{P}^a \pmb{\Phi} \pmb{F}^{\overline{\pi}}$.
To prove the theorem, we have to show that $\pmb{P}^a \pmb{\Phi}  =  \pmb{\Phi} \pmb{P}_\phi^a$ holds.
By applying the identities of Lemma~\ref{lem:abs-mat}, we observe that $\pmb{F}^a$ is a SF with respect to the transition matrix $\pmb{\Omega} \pmb{P}^a \pmb{\Phi}$:
\begin{align}
& \pmb{\Phi} \pmb{F}^a = \pmb{\Phi} + \gamma \pmb{P}^a \pmb{\Phi} \pmb{F}^{\overline{\pi}} \label{eq:sf-proof} \\
\iff & \pmb{\Omega} \pmb{\Phi} \pmb{F}^a = \pmb{\Omega} \pmb{\Phi} + \gamma \pmb{\Omega} \pmb{P}^a \pmb{\Phi} \pmb{F}^{\overline{\pi}} \\
% \iff & \pmb{F}^a = \pmb{I} + \gamma \pmb{\Omega} \pmb{P}^a \pmb{\Phi} \pmb{F}^{\overline{\pi}} \\
\iff & \pmb{F}^a = \pmb{I} + \gamma \pmb{P}_\phi^a \pmb{F}^{\overline{\pi}} \label{eq:phi-sf}
\end{align}
Hence, $\pmb{F}^{\overline{\pi}}$ is a SF and thus has an inverse $\pmb{F}^{\overline{\pi}} = ( \pmb{I} - \gamma \pmb{P}_\phi^a )^{-1}$, because by Lemma~\ref{lem:abs-mat} $\pmb{P}_\phi^a$ is a stochastic matrix and thus $\pmb{I} - \gamma \pmb{P}_\phi^a$ is invertible.
Substituting~\eqref{eq:phi-sf} into~\eqref{eq:sf-proof}:
\begin{align}
 \pmb{\Phi} \left( \pmb{I} + \gamma \pmb{P}_\phi^a \pmb{F}^{\overline{\pi}} \right) &= \pmb{\Phi} + \gamma \pmb{P}^a \pmb{\Phi} \pmb{F}^{\overline{\pi}} \\
\iff  \pmb{\Phi} \pmb{P}_\phi^a \pmb{F}^{\overline{\pi}} \left[  \pmb{F}^{\overline{\pi}} \right]^{-1} &= \pmb{P}^a \pmb{\Phi} \pmb{F}^{\overline{\pi}} \left[  \pmb{F}^{\overline{\pi}} \right]^{-1} \\
\iff  \pmb{\Phi} \pmb{P}_\phi^a &= \pmb{P}^a \pmb{\Phi} .
\end{align}
\end{proof}

The proof of Theorem~\ref{thm:approx-model} does not depend on the previous results, because the feature matrix $\pmb{\Phi}$ is a real-valued matrix.

\begin{proof}[Proof of Theorem~\ref{thm:approx-model}]
The value error is bounded by 
\begin{align*}
\norm{\pmb{\Phi} \pmb{q}^a_\phi - \pmb{q}^a}_\infty = \norm{ \pmb{\Phi} \pmb{r}_\phi^a + \gamma \pmb{\Phi} \pmb{P}_\phi^a \pmb{v}^\pi_\phi - \pmb{r}^a - \gamma \pmb{P}^a \pmb{v}^\pi }_\infty & \\
\le \norm{ \pmb{\Phi} \pmb{r}_\phi^a - \pmb{r}^a  }_\infty + \gamma \norm{\pmb{\Phi} \pmb{P}_\phi^a \pmb{v}^\pi_\phi - \pmb{P}^a \pmb{v}^\pi}_\infty &
\end{align*}
The second transition error term can be upper bounded with
\begin{align}
& \norm{\pmb{\Phi} \pmb{P}_\phi^a \pmb{v}^\pi_\phi - \pmb{P}^a \pmb{\Phi} \pmb{v}^\pi_\phi + \pmb{P}^a \pmb{\Phi} \pmb{v}^\pi_\phi - \pmb{P}^a \pmb{v}^\pi}_\infty \nonumber \\
&\le \norm{\pmb{\Phi} \pmb{P}_\phi^a \pmb{v}^\pi_\phi - \pmb{P}^a \pmb{\Phi} \pmb{v}^\pi_\phi}_\infty + \norm{ \pmb{P}^a \pmb{\Phi} \pmb{v}^\pi_\phi - \pmb{P}^a \pmb{v}^\pi}_\infty \nonumber  \\
&\le \norm{\pmb{\Phi} \pmb{P}_\phi^a - \pmb{P}^a \pmb{\Phi} }_\infty \norm{\pmb{v}^\pi_\phi}_\infty + \norm{\pmb{P}^a}_\infty \norm{ \pmb{\Phi} \pmb{v}^\pi_\phi - \pmb{v}^\pi}_\infty \nonumber  \\
&\le \norm{\pmb{\Phi} \pmb{P}_\phi^a - \pmb{P}^a \pmb{\Phi} }_\infty \norm{\pmb{v}^\pi_\phi}_\infty + \norm{ \pmb{\Phi} \pmb{v}^\pi_\phi - \pmb{v}^\pi}_\infty. \label{eq:P-err-bnd-1}
\end{align}
The norm of $\pmb{v}_\phi^\pi$ can be upper bounded with
\begin{align*}
 & \norm{\pmb{v}^\pi_\phi}_\infty \le \norm{ \sum_{t=1}^\infty \gamma^{t-1} \left[ \pmb{P}_\phi^\pi \right]^{t-1} \pmb{r}_\phi^\pi }_\infty \\
 &\le \sum_{t=1}^\infty \gamma^{t-1} \norm{ \left[ \pmb{P}_\phi^\pi \right]^{t-1} }_\infty \norm{ \pmb{r}_\phi^\pi }_\infty \le \sum_{t=1}^\infty \gamma^{t-1}\norm{ \pmb{r}_\phi^\pi }_\infty .
\end{align*} 
To bound the term $\gamma \norm{\pmb{\Phi} \pmb{P}_\phi^a - \pmb{P}^a \pmb{\Phi} }_\infty$, we note that $\pmb{F}^{\overline{\pi}}$ has an inverse, because $\pmb{F}^{\overline{\pi}} = \left( \pmb{I} - \gamma \pmb{P}^{\overline{\pi}}_\phi \right)^{-1}$. % and thus $\left[ \pmb{F}^{\overline{\pi}} \right]^{-1} = \left( \pmb{I} - \gamma \pmb{P}^{\overline{\pi}}_\phi \right)$.
The norm of $\left[ \pmb{F}^{\overline{\pi}} \right]^{-1}$ can be bounded with
\begin{equation*}
\norm{ \left[ \pmb{F}^{\overline{\pi}} \right]^{-1} }_\infty = \norm{ \pmb{I} - \gamma \pmb{P}^{\overline{\pi}}_\phi }_\infty \le 1 +  \gamma \norm{\pmb{P}^{\overline{\pi}}_\phi }_\infty \le 1 + \gamma .
\end{equation*}
Hence we can write
\begin{align}
&\pmb{\Phi} + \gamma \pmb{P}^a \pmb{\Phi} \pmb{F}^{\overline{\pi}} - \pmb{\Phi} \pmb{F}^a \nonumber \\
&= \pmb{\Phi} + \gamma \pmb{P}^a \pmb{\Phi} \pmb{F}^{\overline{\pi}} - \pmb{\Phi} \left( \pmb{I} + \gamma \pmb{P}_\phi^a \pmb{F}^{\overline{\pi}} \right) \nonumber \\
&= \gamma \left( \pmb{P}^a \pmb{\Phi} - \pmb{\Phi} \pmb{P}_\phi^a \right) \pmb{F}^{\overline{\pi}} .
\end{align}
Hence 
\begin{equation*}
\left( \pmb{\Phi} + \gamma \pmb{P}^a \pmb{\Phi} \pmb{F}^{\overline{\pi}} - \pmb{\Phi} \pmb{F}^a \right) \left[ \pmb{F}^{\overline{\pi}} \right]^{-1} = \gamma \left( \pmb{P}^a \pmb{\Phi} - \pmb{\Phi} \pmb{P}_\phi^a \right),
\end{equation*}
and thus the transition norm term can be bounded with
\begin{align}
&\norm{ \gamma \left( \pmb{P}^a \pmb{\Phi} - \pmb{\Phi} \pmb{P}_\phi^a \right) }_\infty \nonumber \\
&= \norm{ \left( \pmb{\Phi} + \gamma \pmb{P}^a \pmb{\Phi} \pmb{F}^{\overline{\pi}} - \pmb{\Phi} \pmb{F}^a \right) \left[ \pmb{F}^{\overline{\pi}} \right]^{-1} }_\infty \nonumber \\
&\le \norm{ \pmb{\Phi} + \gamma \pmb{P}^a \pmb{\Phi} \pmb{F}^{\overline{\pi}} - \pmb{\Phi} \pmb{F}^a }_\infty \norm{ \left[ \pmb{F}^{\overline{\pi}} \right]^{-1} }_\infty \nonumber  \\
&\le \varepsilon_\psi \left( 1 + \gamma \right) . \label{eq:P-err-bnd-2}
\end{align}
Using~\eqref{eq:P-err-bnd-2} and~\eqref{eq:P-err-bnd-1} the value function approximation error can be upper bounded with
\begin{align}
&\norm{\pmb{\Phi} \pmb{q}^a_\phi - \pmb{q}^a}_\infty \nonumber \\
&\le \norm{ \pmb{\Phi} \pmb{r}_\phi^a - \pmb{r}^a  }_\infty + \gamma \norm{\pmb{\Phi} \pmb{P}_\phi^a \pmb{v}^\pi_\phi - \pmb{P}^a \pmb{v}^\pi}_\infty \nonumber \\
&\le \varepsilon_r + \gamma \norm{\pmb{\Phi} \pmb{P}_\phi^a - \pmb{P}^a \pmb{\Phi} }_\infty \norm{\pmb{v}^\pi_\phi}_\infty + \gamma \norm{ \pmb{\Phi} \pmb{v}^\pi_\phi - \pmb{v}^\pi}_\infty \nonumber \\
&\le \varepsilon_r + \frac{\varepsilon_\psi \left( 1 + \gamma \right) \norm{\pmb{r}_\phi}_\infty }{1 - \gamma} + \gamma \norm{ \pmb{\Phi} \pmb{v}^\pi_\phi - \pmb{v}^\pi}_\infty . \label{eq:thm-esp-bnd-1}
\end{align}
The bound~\eqref{eq:thm-esp-bnd-1} is independent of the action selected, hence we have that if $B>0$ is an upper bound for all action-value errors, i.e. $\norm{\pmb{\Phi} \pmb{q}^a_\phi - \pmb{q}^a}_\infty \le B$, then $\norm{ \pmb{\Phi} \pmb{v}^\pi_\phi - \pmb{v}^\pi}_\infty \le B$.
This allows us to write the final bound:
\begin{align*}
\norm{\pmb{\Phi} \pmb{q}^a_\phi - \pmb{q}^a}_\infty \le B &= \varepsilon_r + \frac{\varepsilon_\psi \left( 1 + \gamma \right) \norm{\pmb{r}_\phi}_\infty }{1 - \gamma} + \gamma B \\
\iff ~~~ \norm{\pmb{\Phi} \pmb{q}^a_\phi - \pmb{q}^a}_\infty &\le \frac{\varepsilon_r}{1 - \gamma} + \frac{\varepsilon_\psi \left( 1 + \gamma \right) \norm{\pmb{r}_\phi}_\infty }{(1 - \gamma)^2}.
\end{align*}
\end{proof}

\bibliography{library}
\bibliographystyle{icml2018}

\end{document}